\documentclass[pdflatex,sn-mathphys-ay]{sn-jnl}

\usepackage{graphicx}%
\usepackage{multirow}%
\usepackage{amsmath,amssymb,amsfonts}%
\usepackage{amsthm}%
\usepackage{mathrsfs}%
\usepackage[title]{appendix}%
\usepackage{xcolor}%
\usepackage{textcomp}%
\usepackage{manyfoot}%
\usepackage{booktabs}%
\usepackage{algorithm}%
\usepackage{algorithmicx}%
\usepackage{algpseudocode}%
\usepackage{listings}%

\usepackage{mathtools}

\theoremstyle{thmstyleone}%
\newtheorem{theorem}{Theorem}
\theoremstyle{thmstyletwo}%

\theoremstyle{thmstylethree}%
\newtheorem{definition}{Definition}%

\newtheorem{lemma}[theorem]{Lemma}
\newcommand{\Kset}{\mathcal{K}}
\newcommand{\Fset}{\mathcal{F}}
\newcommand{\duoE}{\mathbb{E}}

\newcommand{\Real}{\mathbb{R}}

\begin{document}

\title[Article Title]{Adversarial Bandit Optimization with Globally Bounded Perturbations to Linear Losses}


\author*[1]{\fnm{Zhuoyu} \sur{Cheng}}\email{cheng.zhuoyu.874@s.kyushu-u.ac.jp}

\author[2,3]{\fnm{Kohei} \sur{Hatano}}\email{hatano@inf.kyushu-u.ac.jp}

\author[2]{\fnm{Eiji} \sur{Takimoto}}\email{eiji@inf.kyushu-u.ac.jp}

\affil*[1]{\orgdiv{Joint Graduate School of Mathematics for Innovation}, \orgname{Kyushu University}, \country{Japan}}

\affil[2]{\orgdiv{Department of Informatics}, \orgname{Kyushu University}, \country{Japan}}

\affil[3]{\orgname{RIKEN AIP}, \country{Japan}}


\abstract{We study a class of adversarial bandit optimization problems in which the loss functions may be non-convex and non-smooth.
In each round, the learner observes a loss that consists of an underlying linear component together with an additional perturbation applied after the learner selects an action.
The perturbations are measured relative to the linear losses and are constrained by a global budget that bounds their cumulative magnitude over time.

Under this model, we establish both expected and high-probability regret guarantees.
As a special case of our analysis, we recover an improved high-probability regret bound for classical bandit linear optimization, which corresponds to the setting without perturbations.
We further complement our upper bounds by proving a lower bound on the expected regret.}




\maketitle

\section{Introduction}\label{sec1}
Bandit optimization can be viewed as a sequential interaction between a learner
and an adversary. The interaction unfolds over a finite horizon of $T$ rounds.
The problem is specified by a pair $(\mathcal{K}, \mathcal{F})$, where
$\mathcal{K} \subseteq \mathbb{R}^d$ denotes a bounded, closed, and convex action
set, and $\mathcal{F}$ is a class of loss functions mapping from $\mathcal{K}$ to
$\mathbb{R}$.
At each round $t \in [T]$, the learner selects an action $x_t \in \mathcal{K}$.
Subsequently, the adversary specifies a loss function $f_t \in \mathcal{F}$, and
the learner observes only the scalar loss value $f_t(x_t)$. The loss function
$f_t$ itself is not revealed to the learner. Throughout this paper, we restrict
attention to an oblivious adversary, meaning that the sequence of loss functions
$\{f_t\}_{t=1}^T$ is fixed prior to the start of the game.\footnote{We do not
consider adaptive adversaries that may select $f_t$ based on the learner’s past
actions $x_1,\ldots,x_{t-1}$.}
The performance of the learner is measured by the regret, defined as
\begin{align}
\sum_{t=1}^{T} f_t(x_t)
- \min_{x \in \mathcal{K}} \sum_{t=1}^{T} f_t(x),
\end{align}
which can be studied either in expectation or with high probability.

Bandit optimization with convex loss functions has been extensively investigated.
Early work by \cite{flaxman-etal:soda05} established regret bounds of order
$\mathcal{O}(d^{1/3} T^{3/4})$, while later results improved the dependence on the
horizon, including the information-theoretic bound
$\widetilde{\mathcal{O}}(d^{2.5}\sqrt{T})$ due to \cite{lattimore:msl20}. In the
special case of linear losses, \cite{abernethy2008competing} introduced the
SCRiBLe algorithm and obtained an expected regret bound of
$\mathcal{O}(d\sqrt{T\ln T})$, which is optimal in its dependence on $T$
\citep{bubeck2012towards}. High-probability guarantees for bandit linear
optimization were later established by \cite{lee2020bias}, who proposed a variant
of SCRiBLe with lifting and increasing learning rates and derived a regret bound of
$\widetilde{\mathcal{O}}(d^2\sqrt{T})$.

More recently, attention has also turned to bandit optimization with non-convex
loss functions. For instance, \cite{agarwal2019learning} studied smooth and
bounded non-convex losses and obtained regret bounds of order
$\mathcal{O}(\mathrm{poly}(d) T^{2/3})$. Related results under alternative
structural assumptions were presented by \cite{ghai2022non}, where non-convex
losses are assumed to admit a reparameterization into convex functions.

In this work, we focus on a class of bandit optimization problems with non-convex
and non-smooth loss functions that can be viewed as perturbations of linear
functions. More precisely, each loss function can be decomposed into a linear
term and an additive non-convex perturbation. The perturbations may be chosen
adversarially and are constrained by a total budget over the horizon $T$.
Specifically, the cumulative magnitude of the perturbations over the $T$ rounds
is bounded by a constant $C$.
This model is
fundamentally different from stochastic linear bandit settings
(e.g.,\cite{abbasi2011improved,amani2019linear}), and standard techniques for
handling stochastic noise are not directly applicable.

We develop a new approach for analyzing high-probability regret in this setting,
based on a novel decomposition of the regret.
Our main contributions can be summarized as follows:
\begin{enumerate}
    \item For $C \neq 0$, we differ from the analysis of SCRiBLe with
    lifting and increasing learning rates \citep{lee2020bias}. Instead, we work
    within a more general SCRiBLe framework \citep{abernethy2008competing} and
    derive both expected and high-probability regret bounds of
    \[
    \widetilde{\mathcal{O}}\!\left(d\sqrt{T} + d\sqrt{TC} + dC\right).
    \]

    \item For $C = 0$, the problem reduces to classical bandit linear
    optimization.  We obtain an
    improved high-probability regret bound compared to prior work
    \citep{lee2020bias}. In particular, for confidence level $1-\gamma$, our bound
    is
    \[
    \mathcal{O}\!\left(
    d\sqrt{T\ln T}
    + \ln T \sqrt{T\ln\!\left(\tfrac{\ln T}{\gamma}\right)}
    + \ln\!\left(\tfrac{\ln T}{\gamma}\right)
    \right).
    \]

    \item Finally, we establish a lower bound of order $\Omega(C)$ on the expected
    regret, demonstrating that the dependence on the perturbation budget $C$ in
    our upper bounds is unavoidable.
\end{enumerate}

\subsection{Relation to the previous version}
This manuscript is a revised and extended version of a previous conference publication \citep{cheng2025adversarial}. Compared to that version, the current paper extends the setting in several important ways. In the earlier work, the perturbation in each round was bounded by a constant $\epsilon$; here, we consider a more general setting where the sum of the absolute values of perturbations across $T$ rounds is bounded by a constant $C$, allowing for particularly large perturbations in some rounds.
Furthermore, we no longer rely on the more complex SCRiBLe with lifting algorithm \citep{cheng2025adversarial}, and instead employ the standard SCRiBLe algorithm \citep{abernethy2012interior}. This allows us to analyze regret directly in the original dimension $d$, without lifting the problem to $d+1$ dimensions. In addition, the previous Assumption 1 in \cite{cheng2025adversarial} is no longer required, as it imposed severe restrictions on the practical applicability of the analysis. Consequently, our formulation and results are more general and practically feasible.
Finally, we develop a new analytical approach to derive regret bounds, redesign the experiments.

\section{Related Work}
Bandit linear optimization was originally introduced by~\cite{awerbuch2004adaptive}, 
who studied the problem under oblivious adversaries and derived a regret bound of 
$\mathcal{O}(d^{3/5}T^{2/3})$. Subsequently, \cite{mcmahan2004online} extended the analysis to adaptive adversaries 
and established a regret bound of $\mathcal{O}(dT^{3/4})$.
A major research direction in bandit optimization focuses on gradient estimation obtained through smoothing-based methods.
In particular, \cite{flaxman-etal:soda05} proposed a seminal approach for constructing
unbiased gradient estimators from bandit feedback.
Building on this framework, \cite{abernethy2008competing} introduced the SCRiBLe
algorithm and proved an expected regret bound of
$\mathcal{O}(d\sqrt{T\ln T})$ against oblivious adversaries.
Later, \cite{bartlett2008high} obtained a high-probability regret bound of
$\mathcal{O}(d^{2/3}\sqrt{T\ln(dT)})$ under additional structural assumptions.
More recently, \cite{lee2020bias} developed a refined variant of SCRiBLe and derived a
high-probability regret bound of $\widetilde{O}(d^2\sqrt{T})$ that applies to both
oblivious and adaptive adversaries.
Beyond these classical results, \cite{ito2023best,ito2023exploration} proposed bandit
linear algorithms capable of adapting between stochastic and adversarial regimes,
achieving $\mathcal{O}(d\sqrt{T\ln T})$ regret in adversarial environments.
Additionally, \cite{rodemann2024reciprocal} explored a conceptual connection between
bandit optimization and Bayesian black-box optimization, providing a unified
perspective on regret guarantees across the two settings.

Unlike bandit convex optimization, which has been extensively studied, bandit
optimization with non-convex loss functions presents additional challenges due to the
difficulty of balancing exploration and exploitation in non-convex domains.
\cite{gao2018online} investigated non-convex losses in conjunction with non-stationary
data and obtained a regret bound of $O(\sqrt{T + poly(T)})$.
Similarly, \cite{yang2018optimal} established an $O(\sqrt{T\ln T})$ regret bound for
non-convex losses.
However, both works rely on smoothness assumptions, whereas the loss functions
considered in our setting are neither convex nor smooth.

\subsection{Comparison to SCRiBLe-Based Methods}

The original SCRiBLe algorithm \citep{abernethy2012interior} does not provide
high-probability regret guarantees.
To address this limitation, \cite{lee2020bias} introduced a more elaborate variant of
SCRiBLe, incorporating lifting techniques and increasing learning rates, together with
a substantially more involved regret analysis.

In contrast, we show that a high-probability regret bound can be obtained for the
standard SCRiBLe  algorithm \citep{abernethy2012interior} with only a minor modification (see Algorithm~\ref{alg1}) in the
setting of oblivious adversaries.
Below, we summarize the main distinctions between our approach and existing
SCRiBLe-based methods:

\begin{enumerate}
    \item \textbf{High-probability guarantees for standard SCRiBLe.}
    We derive a high-probability regret upper bound for the original SCRiBLe algorithm \citep{abernethy2012interior}, whereas prior results required a more sophisticated
    algorithmic variant with additional mechanisms \citep{lee2020bias}.

    \item \textbf{Simplified regret analysis.}
    While the analysis in \cite{lee2020bias} introduces significant technical
    complexity in the oblivious setting, our regret decomposition leads to a more
    streamlined and transparent proof.

    \item \textbf{Improved dependence on dimension.}
    The regret bound obtained in \cite{lee2020bias} exhibits a stronger dependence on
    the dimension $d$, whereas our analysis exhibits a substantially reduced dependence
    on $d$, as highlighted in the introduction for the special case
    $C=0$.

    \item \textbf{Generality of the problem formulation.}
    Similar in spirit to SCRiBLe \citep{abernethy2012interior}, our framework treats
    bandit linear optimization as a special instance of a more general class of
    problems.
\end{enumerate}

\section{Preliminaries}
 This section introduces some necessary notations and defines a $C$-approximately linear function sequence. Then we give our problem setting.

\subsection{Notation}
We abbreviate the $2$-norm $\| \cdot \|_2$ as $\|\cdot\|$. 
For a twice differentiable convex function $\mathcal{R}: \Real^d \to \Real$ and 
any $x, h\in \Real^d$, let  
$\| h \|_{x}=
\| h \|_{\nabla^{2}\mathcal{R}(x)}=\sqrt{h^{\top}\nabla^{2}\mathcal{R}(x)h}$, 
and 
$\| h \|_{x}^{*}=
\| h \|_{(\nabla^{2}\mathcal{R}(x))^{-1}}=\sqrt{h^{\top}(\nabla^{2}\mathcal{R}(x))^{-1}h}$, 
respectively.

Let $\mathbb{S}^d_{1}=\{x\mid\| x \| = 1\}$ and $\mathbb{B}^d_{1}=\{x\mid\| x \| \leq 1\}$.
The vector $e_{i} \in \Real^d$ is a standard basis vector with a value of $1$ in the 
$i$-th position and $0$ in all other positions.
$I$ is an identity matrix with dimensionality implied by context.
$\mathbb{E}_t[\cdot]$ denotes the conditional expectation
given the history up to round $t-1$.

\subsection{Problem Setting}
Let $\mathcal{K} \subseteq \mathbb{R}^{d}$ be a bounded, closed, and convex set
that is centrally symmetric, i.e., for any $x \in \mathcal{K}$, we also have
$-x \in \mathcal{K}$.
Assume further that $\mathcal{K}$ has diameter at most $D$, namely,
$\|x - y\| \le D$ for all $x, y \in \mathcal{K}$.
Moreover, we assume that $\mathcal{K}$ contains the unit ball.
For any $\delta \in (0,1)$, define the shrunk set
\(
\mathcal{K}_{\delta}
=
\bigl\{ x \in \mathbb{R}^{d} \;\big|\; \tfrac{1}{1-\delta} x \in \mathcal{K} \bigr\}.
\)

\begin{definition}[$C$-approximately linear function sequence]\label{def1}
A sequence of functions $\{f_t\}_{t=1}^T$ with
$f_t : \mathcal{K} \to \mathbb{R}$ is called
\emph{$C$-approximately linear}
if there exists a sequence
$\{\theta_t\}_{t=1}^T \subset \mathbb{R}^d$
such that
\[
\sup_{x_1,\ldots,x_T \in \mathcal{K}}
\sum_{t=1}^T
\left| f_t(x_t) - \theta_t^\top x_t \right|
\leq C.
\]
\end{definition}
\noindent
For convenience, define
\(
\sigma_t(x) = f_t(x) - \theta_t^\top x.
\)
Then, by the above definition,
\[
\sum_{t=1}^{T} \bigl| \sigma_t(x_t) \bigr|
\le C,
\quad \forall x_1,\ldots,x_T \in \mathcal{K}.
\]

In this paper, we consider the bandit optimization problem $(\Kset, \Fset)$, where
$\Fset$ is the class of loss functions of the form
$f(x) = \theta^\top x + \sigma(x)$ with $\|\theta\| \le G$.
Across $T$ rounds, the adversary selects a $C$-approximately linear function sequence
$\{f_t\}_{t=1}^T \subset \Fset$. 
We assume $\frac{C}{T} \le \frac{2}{3}$, which ensures that the linear component remains
dominant in each loss (thus keeping the problem within the approximately linear
regime), and also guarantees that the regret upper bound established in
Lemma~\ref{lemm:control-h} holds.
We also make a standard technical assumption that $|f(x)| \le 1$ for all $x \in \mathcal{K}$.
This assumption is imposed purely for convenience and is not essential:
our results extend directly to the case where $|f(x)|$ is bounded by any finite constant.

Bandit optimization with a $C$-approximately linear function sequence is defined as follows.
At each round $t = 1, \ldots, T$, the player selects an action
$x_t \in \mathcal{K}$ without knowing the loss function in advance.
The environment, modeled as an oblivious adversary, chooses a sequence of linear loss
vectors $\{\theta_t\}_{t=1}^T \subset \mathbb{R}^d$ before the interaction begins and
independently of the player's actions.
After the player selects $x_t$, the perturbations $\sigma_t(x_t)$ are chosen adversarially,
with the only restriction being the cumulative magnitude
\(
\sum_{t=1}^T |\sigma_t(x_t)| \le C.
\)
The player then observes only the scalar loss
\(
f_t(x_t) = \theta_t^\top x_t + \sigma_t(x_t).
\)
The goal of the player is to minimize the regret
\(
\sum_{t=1}^T f_t(x_t) 
- \min\limits_{x \in \mathcal{K}} \sum_{t=1}^T f_t(x)
\).

\section{Main Results}
In this section, we first introduce SCRiBLe \citep{abernethy2008competing}, followed by presenting the main contributions of this paper with detailed explanations.

\subsection[SCRiBLe]{SCRiBLe \citep{abernethy2008competing}}

When $C=0$, our problem reduces to bandit linear optimization, which has been
widely studied in the literature.
To introduce our results, we briefly revisit the SCRiBLe algorithm
\citep{abernethy2008competing}, which is designed for bandit linear optimization.
SCRiBLe utilizes a $\nu$-self-concordant barrier $\mathcal{R}$, which always exists
on the action set $\mathcal{K}$.
At each round $t$, the player selects a point $x_t \in \mathcal{K}$, while the
actual point played is
\(
y_t = x_t + \mathbf{A}_t \mu_t,
\)
where $\mathbf{A}_t = [\nabla^2 \mathcal{R}(x_t)]^{-1/2}$ and
$\mu_t$ is sampled uniformly at random from the $d$-dimensional unit sphere
$\mathbb{S}^d_1$.
The algorithm maintains a sequence $\{x_t\}_{t=1}^T \subset \mathcal{K}$, which is
updated according to
\(
x_{t+1}
= \arg\min\limits_{x \in \mathcal{K}}
\left\{
\eta \sum_{\tau=1}^t g_\tau^\top x + \mathcal{R}(x)
\right\},
\)
where $\eta$ is the learning rate and $g_\tau$ is an estimator of $\theta_\tau$.
Using only the scalar feedback $f_t(y_t) = \theta_t^\top y_t$, SCRiBLe constructs
an unbiased estimator
\(
g_t = d f_t(y_t)\mathbf{A}_t^{-1}\mu_t
\)
of the linear loss vector $\theta_t$.

\begin{algorithm}[h]
    \caption{SCRiBLe}
    \label{alg1}
    \begin{algorithmic}[1]
        \Require
        
        $T$, parameters $\eta\in\Real, \delta\in(0, 1)$,
        $\nu$-self-concordant barrier $\mathcal{R}$ on $\mathcal{K}$.
        \State Initialize: $x_1=\arg\min\limits_{x\in\mathcal{K_{\delta}}}{\mathcal{R}(x)}$
        \For{$t=1,..,T$}
            \State let $\mathbf{A}_{t}=[\nabla^2\mathcal{R}(x_{t})]^{-\frac{1}{2}}$ \State  Draw $\mu_{t}$ from $\mathbb{S}_{1}^{d}$ uniformly, set $y_t=x_{t}+\mathbf{A}_{t}\mu_{t}$.  
            \State Play $y_{t}$, observe and incur loss $f_{t}(y_{t})$. Let $g_{t}=df_{t}(y_{t})\mathbf{A}_{t}^{-1}\mu_{t}$.
            \State Update $x_{t+1}=\arg\min\limits_{x\in\mathcal{K_{\delta}}}{\eta\sum_{\tau=1}^{t}g_{\tau}^{\top}x+\mathcal{R}(x)}$
        \EndFor
    \end{algorithmic}
\end{algorithm}
\noindent
For the decision set $\mathcal{K}$, we apply Algorithm~\ref{alg1} to a
$C$-approximately linear function sequence.

Algorithm~\ref{alg1} differs only minimally from SCRiBLe
\citep{abernethy2008competing}, except that the points $x_t$ are chosen from the
shrunk set $\mathcal{K}_\delta$ instead of $\mathcal{K}$.
This restriction ensures that $x_t$ remains sufficiently far from the boundary
of $\mathcal{K}$, thereby preventing the Hessian
$\nabla^2 \mathcal{R}(x_t)$ from blowing up, since the $\nu$-self-concordant barrier
$\mathcal{R}$ diverges as $x$ approaches the boundary of $\mathcal{K}$.
The restriction to $\mathcal{K}_\delta$ is required solely to control error terms
involving $\sigma_t(y_t)$.
When the loss functions are purely linear, this restriction becomes unnecessary,
and the standard SCRiBLe algorithm can be applied directly.

Algorithm~\ref{alg1} preserves the main advantages of SCRiBLe.
In particular, a $\nu$-self-concordant barrier $\mathcal{R}$ always exists on
$\mathcal{K}$, allowing us to exploit its geometric properties.
When the loss takes the form $f_t(y_t) = \theta_t^\top y_t$ (i.e., when $C=0$),
the estimator $g_t$ is unbiased for $\theta_t$.
However, in our setting, the estimator
\(
g_t = d \bigl( \theta_t^\top y_t + \sigma_t(y_t) \bigr)\mathbf{A}_t^{-1}\mu_t
\)
is inevitably influenced by the perturbation term $\sigma_t(y_t)$.
As a consequence, our analysis must carefully control terms of the form
\(
(d\sigma_t(y_t)\mathbf{A}_t^{-1}\mu_t)^\top (x_t - h)
\),
where $h \in \mathcal{K}$.
Although the Cauchy--Schwarz inequality yields
\(
(d\sigma_t(y_t)\mathbf{A}_t^{-1}\mu_t)^\top (x_t - h)
\le
\| d\sigma_t(y_t)\mathbf{A}_t^{-1}\mu_t \|^{*}_{\nabla^2\mathcal{R}(x_t)}
\| x_t - h \|_{\nabla^2\mathcal{R}(x_t)},
\)
together with the bound
\(
\| d\sigma_t(y_t)\mathbf{A}_t^{-1}\mu_t \|^{*}_{\nabla^2\mathcal{R}(x_t)}
\le d|\sigma_t(y_t)|
\)
(see inequality~\ref{eq:boundM2}),
the original SCRiBLe analysis \citep{abernethy2008competing} does not provide a way to control
$\| x_t - h \|_{\nabla^2\mathcal{R}(x_t)}$.
Moreover, since the largest eigenvalue of
$\nabla^2\mathcal{R}(x_t)$ may diverge near the boundary
\citep{nemirovski2004interior},
bounding this term poses a significant analytical challenge.



We present our main results: expected and high-probability regret bounds for the problem.
\begin{theorem}
\label{Theo:expected-regret}
For any $\delta \in (0,1)$, the algorithm with parameter
\(
\eta = \frac{\sqrt{\nu \ln \frac{1}{\delta}}}{2d\sqrt{T}}
\)
guarantees the following expected regret bound:
\begin{equation}
\mathbb{E}\Bigg[
\sum_{t=1}^{T} f_t(y_t)
- \min_{x\in \mathcal{K}} \sum_{t=1}^{T} f_t(x)
\Bigg]
\leq
4d\sqrt{\nu T \ln \frac{1}{\delta}}
+ 2Cd(\nu + 2\sqrt{\nu})\frac{1-\delta}{\delta}
+ \delta GDT + 2C .
\end{equation}
\end{theorem}

\noindent
In particular, we choose
\[
\delta =
\begin{cases}
\frac{1}{T^{2}}, & C = 0, \\
\frac{C}{T}, & C \neq 0,
\end{cases}
\]
to optimize the order of the regret bound.

\begin{theorem}
\label{Theo:high-probability-regret}
For any $\delta \in (0,1)$, the algorithm with parameter
\(
\eta = \frac{\sqrt{\nu \ln \frac{1}{\delta}}}{2d\sqrt{T}}
\)
ensures that, with probability at least $1-\gamma$,
\begin{equation}
\begin{aligned}
\sum_{t=1}^{T} f_t(y_t)
- \min_{x\in\mathcal{K}} \sum_{t=1}^{T} f_t(x)
\leq\;&
4d\sqrt{\nu T\ln \frac{1}{\delta}}
+ 2Cd(\nu + 2\sqrt{\nu})\frac{1-\delta}{\delta}
+ \delta GDT \\
&+ 2C
+ S\Big(GD\sqrt{8T \ln\frac{S}{\gamma}}
+ 2GD\ln\frac{S}{\gamma}\Big),
\end{aligned}
\end{equation}
where
$S=\lceil \ln  GD \rceil\lceil \ln ((GD)^{2}T) \rceil$.
\end{theorem}

\noindent
In particular, we choose
\[
\delta =
\begin{cases}
\frac{1}{T^{2}}, & C = 0, \\
\frac{C}{T}, & C \neq 0,
\end{cases}
\]
to optimize the order of the high-probability regret bound.

\subsection{Analysis}
We note that high-probability regret guarantees for bandit linear optimization have also been established in prior work, such as \cite{lee2020bias}, through a more intricate analysis based on a lifting technique that augments the problem into a higher-dimensional space before conducting the regret analysis. In contrast, our approach relies on a different and more intuitive regret decomposition carried out directly in the original space, leading to a simpler and more transparent analysis. Importantly, our arguments apply to a more general algorithmic framework \citep{abernethy2008competing} and do not require lifting or dimensional augmentation.

We primarily divide the regret into parts:
\begin{equation}
    Regret = \underbrace{
    \sum_{t=1}^{T}(x_{t}-h)^{\top}\mathbb{E}[g_{t}]
    }_{\textsc{Reg-Term}}
    + \underbrace{
    \sum_{t=1}^{T}(y_{t}-x_{t})^{\top}\theta_{t}
    }_{\textsc{Deviation-Term}}
    + \underbrace{
    \sum_{t=1}^{T}
    \big(d\sigma_{t}(y_{t})\mathbf{A}_{t}^{-1}\mu_{t}\big)^{\top}(x_{t}-h)
    }_{\textsc{Error-Term}},
\end{equation}
where $h\in\mathcal{K}$.

This decomposition differs from that of \cite{lee2020bias}, which expresses the regret as
\begin{equation}
    Regret = \underbrace{
    \sum_{t=1}^{T}(x_{t}-h)^{\top}g_{t}
    }_{\textsc{Reg-Term}}
    + \underbrace{
    \sum_{t=1}^{T}
    \big[y_{t}^{\top}\theta_{t}
    - x_{t}^{\top}g_{t}
    + h^{\top}(g_{t}-\theta_{t})\big]
    }_{\textsc{Deviation-Term}},
\end{equation}
where $h\in\mathcal{K}$.

As a consequence, for an oblivious adversary, our regret analysis does not require controlling the variance of the estimators $g_{t}$ and $\theta_{t}$, but only the deviation induced by the discrepancy between $y_{t}$ and $x_{t}$. This structural difference is a key factor that enables us to obtain an improved high-probability regret bound when $C=0$. In addition, when $C \neq 0$, an extra \textsc{Error-Term} must be accounted for in the analysis.

The bound on the \textsc{Reg-Term} is identical to that in previous work
\citep{abernethy2012interior, lee2020bias}, and we therefore omit the details
and focus on the \textsc{Error-Term}. By the Cauchy--Schwarz inequality, the
\textsc{Error-Term} can be bounded as
\(
\textsc{Error-Term}
\le
\| d\sigma_{t}(y_{t})\mathbf{A}_{t}^{-1}\mu_{t}
\|_{\nabla^2\mathcal{R}(x_{t})}^{*}
\,
\| x_{t}-h \|_{\nabla^2\mathcal{R}(x_{t})}
\),
and moreover
\(
\sum_{t=1}^{T}
\| d\sigma_{t}(y_{t})\mathbf{A}_{t}^{-1}\mu_{t}
\|_{\nabla^2\mathcal{R}(x_{t})}^{*}
\le dC
\).
The main difficulty thus lies in controlling
\(
\| x_{t}- h\|_{\nabla^2\mathcal{R}(x_{t})}
\),
since the largest eigenvalue of
\(
\nabla^2\mathcal{R}(x_{t})
\)
may diverge as \(x_t\) approaches the boundary of
\(\mathcal{K}\) \citep{nemirovski2004interior}, an issue that has not been
adequately addressed in previous work. Using lifting and increasing learning
rates, \cite{lee2020bias} derives the inequality
\(
\| h\|_{\nabla^2\mathcal{R}(x_{t})}
\le
-\| h\|_{\nabla^2\mathcal{R}(x_{t+1})}
+ \nu \ln (\nu T+1)
\),
but this bound is too loose to yield a meaningful bound of
\(
\| x_{t}- h\|_{\nabla^2\mathcal{R}(x_{t})}
\).
In contrast, by restricting the iterates \(x_t\) to the shrunken domain
\(\mathcal{K}_\delta \subset \mathcal{K}\), we obtain a key technical result
(Lemma~\ref{lemm:control-h}) establishing that
\(
\| x_t - h \|_{\nabla^2 \mathcal{R}(x_t)}
\le
2\left(\frac{1}{\delta}-1\right)(\nu + 2\sqrt{\nu})\) with 
\(
 0 < \delta \leq \frac{2}{3}
\).
This bound directly resolves the main difficulty in controlling the
\textsc{Error-Term} and, moreover, shows that lifting and increasing learning
rates are unnecessary in our analysis, as they are mainly used in
\cite{lee2020bias} to control
\(
\| h\|_{\nabla^2\mathcal{R}(x_{t})}
\).

Once we obtain the expected and high-probability bounds for the
\textsc{Deviation-Term}, the corresponding expected and high-probability regret
bounds follow immediately.
In the special case \(C=0\), our analysis of the high-probability bound differs
substantially from the SCRiBLe-based approach in \citep{lee2020bias}.
Their method 
imposes additional constraints on \(\delta\) to ensure that the iterates
\(x_t\) remain sufficiently far from the boundary of \(\mathcal{K}\), which is
needed to control the eigenvalues of \(\mathbf{A}_t\), in particular for bounding
\(\|h\|_{\nabla^2\mathcal{R}(x_t)}\).
In contrast, our approach does not require the iterates to stay away from the
boundary, nor does it rely on bounding the eigenvalues of \(\mathbf{A}_t\).
This allows greater flexibility in the choice of \(\delta\) (e.g.,
\(\delta = 1/T^2\)), which in turn leads to a tighter high-probability regret
bound.

Finally, we prove the lower bound of regret in section 6.


\section{Proof}
This section reviews preliminaries on the $\nu$-self-concordant barrier,
introduces several essential lemmas, and presents the proofs of the main theorems.


\subsection[nu-self-concordant barrier]{$\nu$-self-concordant barrier}

We introduce the $\nu$-self-concordant barrier, providing its definitions and highlighting several key properties that will be frequently used in the subsequent analysis.

\begin{definition}[\cite{nemirovski2004interior,nesterov1994interior}]
     Let $\mathcal{K}\subseteq\Real^{d}$ be a convex set with a nonempty interior $int(\mathcal{K})$. A function $\mathcal{R}:int(\mathcal{K})\to \Real$: is called a $\nu$-self-concordant barrier on $\mathcal{K}$ if 
    \begin{enumerate}
        \item $\mathcal{R}$ is three times continuously differentiable and convex and approaches infinity along any sequence of points approaching the boundary of $\mathcal{K}$. 
        \item For every $h\in\Real^{d}$ and $x\in int(\mathcal{K})$ the following holds: 
    \begin{equation}
        |\sum_{i=1}^{d}\sum_{j=1}^{d}\sum_{k=1}^{d}\frac{\partial^{3}\mathcal{R}(x)}{\partial x_{i}\partial x_{j}\partial x_{k}}h_{i}h_{j}h_{k}|\leq 2 \| h \|^{3}_{x},
    \end{equation}
    \begin{equation}
        |\nabla \mathcal{R}(x)^{\top}h |\leq \sqrt{\nu}\| h \|_{x},
    \end{equation}

    \end{enumerate}
\end{definition}

\begin{lemma}[\cite{nemirovski2004interior}]
\label{lemm:normal-property}
    If $\mathcal{R}$ is a $\nu$-self-concordant barrier on $\mathcal{K}$, then the Dikin ellipsoid centered at $x\in int(\mathcal{K})$, defined as $\{y: \| y-x\|_{x}\leq 1\}$, is always within $\mathcal{K}$. Moreover,
 \begin{equation}
     \| h\|_{y}\geq \| h\|_{x}(1-\| y-x\|_{x})
 \end{equation}
 holds for any $h\in\Real^{d}$ and any $y$ with $\| y-x\|_{x}\leq 1$.
\end{lemma}

\begin{lemma}[\cite{hazan2016introduction}]
\label{lemm:normal-log-property}
        Let $\mathcal{R}$ is a $\nu$-self-concordant barrier over $\mathcal{K}$, then for all $x, z\in int(\mathcal{K}): \mathcal{R}(z)-\mathcal{R}(x)\leq \nu \ln \frac{1}{1-\pi_{x}(z)}$, where $\pi_{x}(z)=\inf\{t\geq 0 : x+t^{-1}(z-x)\in\mathcal{K}\}$.
\end{lemma}

\subsection{Useful lemmas}
In addition to the properties of the $\nu$-self-concordant barrier and its related lemmas, we further present several auxiliary lemmas.

\begin{lemma}
\label{lemm:control-h}
    Assume $\delta\leq \frac{2}{3}$. For Algorithm~\ref{alg1} and for any $x,y \in \mathcal{K_{\delta}}$, it holds that
    \begin{equation}
        \| y-x \|_{x}
        \leq 2(\frac{1}{\delta}-1)
        \left(\nu+2\sqrt{\nu}\right).
    \end{equation}
\end{lemma}
\noindent
Lemma~\ref{lemm:control-h} plays a central role in our analysis, particularly in establishing Lemma~\ref{lemm:bound-regret}.

\begin{lemma}[\cite{abernethy2008competing}]
\label{lemm:close-property}
        $x_{t+1}\in W_{4d\eta}(x_{t})$, where $W_{r}(x)=\{y\in \mathcal{K}: \| y-x\|_{x}< r\}$.
\end{lemma}
\noindent
As in Lemma~6 of the SCRiBLe algorithm \citep{abernethy2008competing}, the next minimizer $x_{t+1}$ is “close” to $x_{t}$.
Specifically, Lemma~\ref{lemm:close-property} implies that
$\| x_{t+1}-x_t \|_{x_t} \le 4d\eta$.
This result will help us bound $g_{t}^{\top} (x_{t}-h)$, where $h\in\mathcal{K}$ (see inequality~\ref{eq:uselemmboundxt}).

\begin{lemma}
\label{lemm:gradient-estimate}
For Algorithm~\ref{alg1}, we have $\duoE_{t}[g_{t}]=\theta_{t}+\duoE_{t}[d\sigma_{t}(y_{t})\mathbf{A}_{t}^{-1}\mu_{t}]$.
\end{lemma}
\noindent
Lemma~\ref{lemm:gradient-estimate} is based on Corollary 6.8. \citep{hazan2016introduction}, but due to the differences in the loss functions, $g_{t}$ is no longer an unbiased estimate of $\theta_{t}$.
By applying the conclusion of Corollary 6.8. \citep{hazan2016introduction}, we derive Lemma~\ref{lemm:gradient-estimate}.

\begin{lemma}
\label{lemm:FTRL}
        For Algorithm~\ref{alg1} and for every $h\in \mathcal{K}$,
        $\sum_{t=1}^{T}g_{t}^{\top}x_{t}-\sum_{t=1}^{T}g_{t}^{\top}h \leq \sum_{t=1}^{T}[g_{t}^{\top}x_{t}-g_{t}^{\top}x_{t+1}]+\frac{1}{\eta}[\mathcal{R}(h)-\mathcal{R}(x_{1})]$.
\end{lemma}
\noindent
The regret bound of FTRL algorithm \citep{hazan2016introduction} states that for every $h\in \mathcal{K}$, $\sum_{t=1}^{T}\nabla_{t}^{\top}x_{t}-\sum_{t=1}^{T}\nabla_{t}^{\top}h \leq \sum_{t=1}^{T}[\nabla_{t}^{\top}x_{t}-\nabla_{t}^{\top}x_{t+1}]+\frac{1}{\eta}[\mathcal{R}(h)-\mathcal{R}(x_{1})]$, where $\nabla_{t}$ represents the gradient of the loss function $f_t$. In our adaptation, we replaced $\nabla_{t}$ with $g_{t}$ and $\mathcal{K}$ with 
$\mathcal{K}_{\delta}$. This modification does not fundamentally alter the original result.
Since the update way $x_{t+1}=\arg\min\limits_{x\in\mathcal{K}_{\delta}} {\eta\sum_{\tau=1}^{t}g_{\tau}^{\top}x+\mathcal{R}(x)}$
satisfied the condition of FTRL algorithm \citep{hazan2016introduction}, we can apply Lemma 5.3. in \cite{hazan2016introduction} to Algorithm~\ref{alg1} as follow.

The lemmas presented above collectively lead to Lemma~\ref{lemm:bound-regret}, which constitutes a key technical result of this paper. In particular, it establishes a bound on $\sum_{t=1}^{T}\theta_{t}^\top x_{t}-\sum_{t=1}^{T}\theta_{t}^\top x^{*}$, where $x^{*}=\arg\min\limits_{x\in \mathcal{K}}\sum_{t=1}^{T}f_{t}(x)$. As a direct consequence of Lemma~\ref{lemm:bound-regret}, it suffices to bound the term
$\sum_{t=1}^{T}\theta_{t}^\top (y_{t}-x_{t})$ in order to bound the regret.
Lemma~\ref{lemm:bound-regret} plays a crucial role in deriving both expected and high-probability regret bounds.

\begin{lemma}
\label{lemm:bound-regret}
    For Algorithm~\ref{alg1}, let $f_{t}(x_t)=\theta_{t}^{\top}x_{t}+\sigma_{t}(x_{t})$ and $x^{*}=\arg\min\limits_{x\in \mathcal{K}}\sum_{t=1}^{T}f_{t}(x)$ and we have
    \begin{equation}
         \sum_{t=1}^{T}\theta_{t}^\top x_{t}-\sum_{t=1}^{T}\theta_{t}^\top x^{*}
                \leq 4\eta d^{2}T+\frac{\nu \ln (\frac{1}{\delta})}{\eta}+ 2Cd(\nu +2\sqrt{\nu })(\frac{1-\delta}{\delta}) + \delta DGT .
    \end{equation}
\end{lemma}
    
\noindent
With the help of Lemma~\ref{lemm:bound-regret}, we are ready to present the proof of Theorem~\ref{Theo:expected-regret}.

\subsection{Proof of Theorem~\ref{Theo:expected-regret}}

\begin{proof} 
    Recall that each loss function can be written as $f(x)=\theta^\top x+\sigma(x)$. 
    We decompose the expected regret as
    \begin{align}
    \mathbb{E}\!\left[
    \sum_{t=1}^{T} f_t(y_t) - \sum_{t=1}^{T} f_t(x^*)
    \right]
    &=
    \mathbb{E}\!\left[
    \sum_{t=1}^{T} \bigl( \theta_t^\top y_t + \sigma_t(y_t) \bigr)
    -
    \sum_{t=1}^{T} \bigl( \theta_t^\top x^* + \sigma_t(x^*) \bigr)
    \right]
    \nonumber\\
    &=
    \mathbb{E}\!\left[
    \sum_{t=1}^{T} \theta_t^\top y_t
    -
    \sum_{t=1}^{T} \theta_t^\top x^*
    \right]
    +
    \mathbb{E}\!\left[
    \sum_{t=1}^{T} \sigma_t(y_t)
    -
    \sum_{t=1}^{T} \sigma_t(x^*)
    \right].
    \end{align}

    We first bound the linear component. Observe that
    \begin{align}
    \mathbb{E}\!\left[
    \sum_{t=1}^{T} \theta_t^\top y_t
    -
    \sum_{t=1}^{T} \theta_t^\top x^*
    \right]
    &=
    \sum_{t=1}^{T} \mathbb{E}[\theta_t^\top y_t]
    -
    \sum_{t=1}^{T} \mathbb{E}[\theta_t^\top x_t]
    +
    \sum_{t=1}^{T} \mathbb{E}[\theta_t^\top x_t]
    -
    \sum_{t=1}^{T} \mathbb{E}[\theta_t^\top x^*].
    \end{align}

    By the law of total expectation, we have
    \begin{align}
    \sum_{t=1}^{T} \mathbb{E}[\theta_t^\top y_t]
    -
    \sum_{t=1}^{T} \mathbb{E}[\theta_t^\top x_t]
    &=
    \sum_{t=1}^{T} \mathbb{E}\!\left[ \theta_t^\top (y_t - x_t) \right]
    \nonumber\\
    &=
    \sum_{t=1}^{T} \mathbb{E}\!\left[
    \mathbb{E}_t \bigl[ \theta_t^\top (y_t - x_t) \bigr]
    \right]
    \nonumber\\
    &=
    \sum_{t=1}^{T} \mathbb{E}\!\left[
    \theta_t^\top \mathbb{E}_t [ \mathbf{A}_t \mu_t ]
    \right]
    \nonumber\\
    &=
    \sum_{t=1}^{T} \mathbb{E}[ \theta_t^\top \mathbf{0} ]
    = 0.
    \end{align}
    
   Therefore,
    \begin{equation}
    \mathbb{E}\!\left[
    \sum_{t=1}^{T} \theta_t^\top y_t
    -
    \sum_{t=1}^{T} \theta_t^\top x^*
    \right]
    =
    \mathbb{E}\!\left[
    \sum_{t=1}^{T} \theta_t^\top x_t
    -
    \sum_{t=1}^{T} \theta_t^\top x^*
    \right].
    \end{equation}

   From Lemma~\ref{lemm:bound-regret}, we obtain
    \begin{align}
    \mathbb{E}\!\left[
    \sum_{t=1}^{T} \theta_t^\top x_t
    -
    \sum_{t=1}^{T} \theta_t^\top x^*
    \right]
    &\le
    4 \eta d^{2} T
    +
    \frac{\nu \ln \!\left( \tfrac{1}{\delta} \right)}{\eta}
    +
    2 C d (\nu + 2 \sqrt{\nu}) \frac{1-\delta}{\delta}
    +
    \delta D G T.
    \end{align} 
\noindent  
Since the perturbation $\sigma_t$ is chosen after observing the player's action,
it may induce a worst-case deviation.
By Definition~\ref{def1}, it follows that
\begin{equation}
\mathbb{E}\!\left[
\sum_{t=1}^{T} \sigma_t(y_t)
-
\sum_{t=1}^{T} \sigma_t(x^*)
\right]
\le 2C .
\label{eq:perturbation-bound}
\end{equation}

Combining this bound with Lemma~\ref{lemm:bound-regret}, we obtain
\begin{align}
\textsc{Regret}
&=
\mathbb{E}\!\left[
\sum_{t=1}^{T} f_t(y_t)
-
\sum_{t=1}^{T} f_t(x^*)
\right]
\nonumber\\
&\le
4 d \sqrt{\nu T \ln \tfrac{1}{\delta}}
+
2 C d (\nu + 2 \sqrt{\nu}) \frac{1-\delta}{\delta}
+
\delta G D T
+
2 C,
\end{align}
where
\(
\eta = \frac{\sqrt{\nu \ln \tfrac{1}{\delta}}}{2 d \sqrt{T}}
\).
\end{proof}

\subsection{Proof of Theorem~\ref{Theo:high-probability-regret}}

To establish the high-probability regret bound, we first introduce the necessary Lemma~\ref{lemm:martingale}.

\begin{lemma}[Theorem 2.2. in~\cite{lee2020bias}]
\label{lemm:martingale}
    Let $X_{1},...,X_{T}$ be a martingale difference sequence with respect to a filtration $F_{1}\subseteq...\subseteq F_{T}$ such that $\duoE[X_{t}\mid F_{t}]=0$. Suppose $B_{t}\in[1, b]$ for a fixed constant $b$ is $F_{t}$-measurable and such that $X_{t}\leq B_{t}$ holds almost surely. Then with probability at least $1-\gamma$
we have
\begin{equation}
    \sum_{t=1}^{T}X_{t}\leq S(\sqrt{8V\ln(S/\gamma)}+2B^{*}\ln(S/\gamma)),
\end{equation}
where $V=\max \{1; \sum_{t=1}^{T}\duoE[X_{t}^{2}\mid F_{t}]\}, B^{*}=\max\limits_{t\in[T]}B_{t}$, and $S=\lceil \log b \rceil\lceil \log(b^{2}T) \rceil$.
\end{lemma}
The analysis in~\cite{lee2020bias} applies Lemma~\ref{lemm:martingale} to obtain a
high-probability bound on
\(
\sum_{t=1}^{T}\bigl[y_{t}^{\top}\theta_{t}-x_{t}^{\top}g_{t}+u^{\top}(g_{t}-\theta_{t})\bigr].
\)
By contrast, we define $X_t=\theta_{t}^{\top}y_{t}-\theta_{t}^{\top}x_{t}$ and directly derive a
high-probability bound on
\(
\sum_{t=1}^{T}\bigl(\theta_{t}^{\top}y_{t}-\theta_{t}^{\top}x_{t}\bigr).
\)
This alternative application of Lemma~\ref{lemm:martingale} leads to a tighter
high-probability upper bound for bandit linear optimization.

With the support of Lemma~\ref{lemm:bound-regret} and Lemma~\ref{lemm:martingale}, we are ready to prove Theorem~\ref{Theo:high-probability-regret}.

\begin{proof}
Define
\(
X_t = \theta_t^\top y_t - \theta_t^\top x_t
\).
By construction and the symmetry of the random perturbation, we have
\(
\mathbb{E}_t[X_t] = 0
\).
Moreover,
\(
|X_t|
\le \|\theta_t\| \, \|y_t - x_t\|
\le G D
\),
and hence
\begin{align}
\mathbb{E}_t[X_t^2]
&=
\mathbb{E}_t\!\left[
(\theta_t^\top y_t - \theta_t^\top x_t)^2
\right]
\le
G^2 D^2 .
\label{eq:xt-variance}
\end{align}

We next bound the regret. By the decomposition
\( f_t(x) = \theta_t^\top x + \sigma_t(x) \), we obtain
\begin{align}
\sum_{t=1}^{T} f_t(y_t) - \sum_{t=1}^{T} f_t(x^*)
&=
\sum_{t=1}^{T} \theta_t^\top y_t
-
\sum_{t=1}^{T} \theta_t^\top x^*
+
\sum_{t=1}^{T} \sigma_t(y_t)
-
\sum_{t=1}^{T} \sigma_t(x^*)
\nonumber\\
&\le
\sum_{t=1}^{T} \theta_t^\top y_t
-
\sum_{t=1}^{T} \theta_t^\top x^*
+
2C
\nonumber\\
&=
\sum_{t=1}^{T} X_t
+
\sum_{t=1}^{T} \theta_t^\top x_t
-
\sum_{t=1}^{T} \theta_t^\top x^*
+
2C.
\label{eq:loss-split}
\end{align}

By Lemma~\ref{lemm:bound-regret}, we have
\begin{align}
\sum_{t=1}^{T} \theta_t^\top x_t
-
\sum_{t=1}^{T} \theta_t^\top x^*
\le
4 \eta d^2 T
+
\frac{\nu \ln \!\left( \tfrac{1}{\delta} \right)}{\eta}
+
2 C d (\nu + 2 \sqrt{\nu}) \frac{1-\delta}{\delta}
+
\delta D G T ,
\label{eq:linear-regret-bound}
\end{align}
where
\(
\eta = \frac{\sqrt{\nu \ln T}}{2 d \sqrt{T}}
\).

Finally, since \(\{X_t\}_{t=1}^T\) forms a martingale difference sequence with
\(
|X_t| \le B^* = b = G D
\)
and
\(
\sum_{t=1}^{T} \mathbb{E}_t[X_t^2] \le V = G^2 D^2 T
\),
Lemma~\ref{lemm:martingale} implies that, with probability at least \(1-\gamma\),
\begin{align}
\sum_{t=1}^{T} X_t
\le
S \Bigl(
\sqrt{8 V \ln (S/\gamma)}
+
2 B^* \ln (S/\gamma)
\Bigr),
\label{eq:martingale-bound}
\end{align}
where
\(
S = \lceil \ln (G D) \rceil \, \lceil \ln ((G D)^2 T) \rceil
\).

Combining \eqref{eq:loss-split}–\eqref{eq:martingale-bound} completes the proof.
\end{proof}

\subsection{Application to black-box optimization}
Through an online-to-offline transformation, our results can also be applied to black-box optimization.
This setting is particularly important because most existing theoretical analyses of black-box optimization are restricted to linear, convex, or smooth objectives in adversarial environments, typically studied under the framework of bandit convex optimization.
Therefore, it is meaningful to investigate black-box optimization problems beyond these structural assumptions.
In contrast, the objective considered in this paper is neither linear nor smooth, even under mild assumptions.

Let $\hat{x}$ denote the output of Algorithm~\ref{alg1}.
It follows from Theorem~1 that
$f(\hat{x}) - \min\limits_{x \in \mathcal{K}} f(x)
\leq \frac{4d\sqrt{\nu \ln \frac{1}{\delta}}}{\sqrt{T}}
+ \delta GD
+ d(\nu + 2\sqrt{\nu})\big(\frac{1-\delta}{\delta}\big)\frac{C}{T}
+ \frac{2C}{T}$.
Moreover, we establish a lower bound of order $\frac{2C}{T}$ for this problem; see Lemma~\ref{lemm:blackbox-lower}.
Consequently, the gap between the upper and lower bounds is
$\frac{4d\sqrt{\nu \ln \frac{1}{\delta}}}{\sqrt{T}}
+ \delta GD
+ d(\nu + 2\sqrt{\nu})\big(\frac{1-\delta}{\delta}\big)\frac{C}{T}$.
This gap indicates the potential existence of intermediate regimes between fully adversarial environments and standard stochastic settings, where tighter guarantees and more efficient algorithms may be possible.
Exploring such regimes is an interesting direction for future research.

\section{Lower bound}
In this section, we establish a lower bound on the regret.
We begin by introducing the following definition.

\begin{definition} [$\epsilon$-approximately linear function]
A function $f: \mathcal{K} \to \Real$ is 
$\epsilon$-approximately linear
if there exists $\theta_f \in \Real^d$ such that 
$\forall x\in\mathcal{K}$, $|f(x) - \theta_f^\top x| \leq \epsilon$. 
\end{definition}
Observe that a sequence of $\epsilon$-approximately linear functions
$\{f_t\}_{t=1}^T$
is a special case of a $C$-approximately linear function sequence,
with cumulative budget $C = \epsilon T$.
Consequently, establishing a regret lower bound for bandit optimization
with $\epsilon$-approximately linear loss functions
immediately yields a regret lower bound
for $C$-approximately linear function sequences.

We consider a black-box optimization problem for 
the set $\Fset$ of $\epsilon$-approximately linear functions $f:\Kset \to \Real$.
In the problem, we are given access to the oracle $O_f$ for some $f\in \Fset$, 
which returns the value $f(x)$ given an input $x\in \Kset$. 
The goal is to find a point $\hat{x} \in \Kset$ such that 
$f(\hat{x}) -\min_{x\in \Kset} f(x)$ is small enough.
Then, the following statement holds.
 
\begin{lemma}
\label{lemm:blackbox-lower}
For any algorithm $\mathcal{A}$ for the black-box optimization problem for $\Fset$,
there exists an $\epsilon$-approximately linear function $f \in \Fset$ such that 
the output $\hat{x}$ of $\mathcal{A}$ satisfies 
\begin{equation}
f(\hat{x})-\min_{x\in\mathcal{K}}f(x) \geq 2\epsilon.
\end{equation}
\end{lemma}

\begin{proof}   
    Firstly, suppose that 
    the algorithm $\mathcal{A}$ is deterministic. 
    At iteration $t=1,...,T$, for any feedback $y_{1},...,y_ {t-1}\in\Real$,  
    $\mathcal{A}$ should choose the next query point $x_{t}$ based on the data observed so far. 
    That is, 
    \begin{equation}
        x_{t}=\mathcal{A}((x_{1}, y_{1}),...,(x_{t-1}, y_{t-1})).
    \end{equation}
    Assume that the final output $\hat{x}$ is returned after $T$ queries to the oracle $O_f$. 
    In particular, we fix the $T$ feedbacks $y_1=y_2=\dots=y_T=\epsilon$.
    Let $z\in\mathcal{K}$ be such that $z\notin\{x_{1},...,x_{T},\hat{x}\}$. 
    Then we define a function $f: \mathcal{K}\to\Real$ is as
    \begin{equation}
    \label{lower-bound-equation}
         f(x)= \begin{cases}
            \epsilon,\quad &x\not= z, \\
            -\epsilon,\quad &x=z.
        \end{cases}  
    \end{equation} 
    The function $f$ is indeed an $\epsilon$-approximately linear function, 
    as $f(x)=0^\top x + \sigma(x)$, where $\sigma(x)=\epsilon$ for $x \neq z$ and 
    $\sigma(x)=-\epsilon$ for $x =z$. 
    Further, we have 
    
    \begin{equation}
        f(\hat{x})-\min_{x\in\mathcal{K}} f(x) \geq 2\epsilon. 
    \end{equation}
    
    Secondly, if algorithm $\mathcal{A}$ is randomized. 
    It means each $x_{t}$ is chosen randomly.  
    We assume the same feedbacks $y_1=y_2=\dots=y_T=\epsilon$.
    Let $X=\{x_{1},...,x_{T},\hat{x}\}$.
    Then, there exists a point $z\in\mathcal{K}$ such that $P_{X}(z\in X)=0$, since 
    $\duoE_{z'}[P_{X}(z'\in X|z')]=P_{z',X}(z'\in X)=\duoE_{X}[P_{z'}(z'\in X|X)] =0$, 
    where the expectation on $z'$ is defined w.r.t.\ the uniform distribution over $\Kset$.
    For the objective function $f$ defined in (~\ref{lower-bound-equation}), we have
    $f(\hat{x})-\min_{x\in\mathcal{K}} f(x) \geq 2\epsilon$ while $f$ is $\epsilon$-approximately linear.
\end{proof}

\begin{theorem}
For any horizon $T \ge 1$ and any player, there exists an adversary generating
a $C$-approximately linear function sequence such that the regret is at least
$2C$.
\end{theorem}
\noindent
In particular, this lower bound applies to sequences of
$\epsilon$-approximately linear functions, for which $C = \epsilon T$.


\begin{proof}
We prove the statement by contradiction. 
Suppose that there exists a player whose regret is 
less than $2C$. 
Then we can construct an algorithm for the black-box optimization problem from it by 
feeding the online algorithm with $T$ feedbacks of the black-box optimization problem and 
by setting $\hat{x} \in \arg\min\limits_{t \in [T]} f(x_t)$. Then, 
\[
f(\hat{x}) - \min_{x \in \Kset}f(x) \leq \frac{\sum_{t=1}^T f(x_t) -\sum_{t=1}^T\min\limits_{x\in \Kset}f(x) }{T} < 2\epsilon, 
\]
which contradicts Lemma~\ref{lemm:blackbox-lower}.
\end{proof}

This lower bound implies that an $\Omega(C)$ regret is unavoidable
for the bandit optimization problem with $C$-approximately linear function sequences.
We conjecture that this lower bound can be further tightened to $\Omega(dC)$,
and leave this as an open problem for future work.

\section{Experiments}
\begin{figure}[h]
\includegraphics[width=0.9\textwidth]{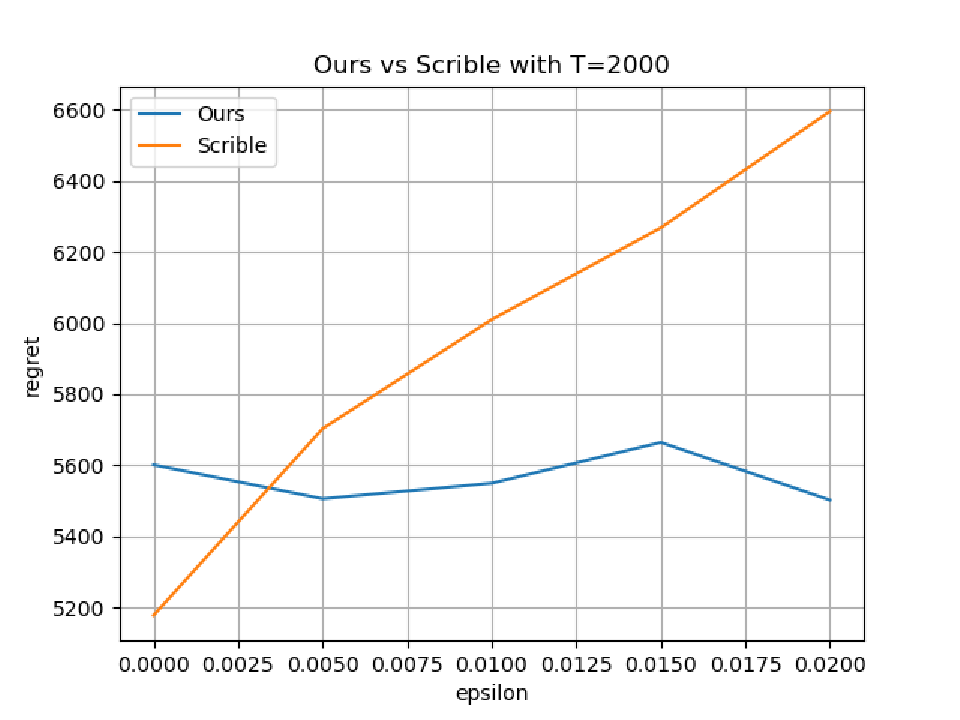}
\caption{Average regret of algorithms for artificial data sets. The yellow line corresponds to the results of Algorithm~\ref{alg1}, the blue line corresponds to the SCRiBLe algorithm~\cite{abernethy2008competing}} \label{fig01}
\end{figure}

We conduct comparative experiments over a range of $C$ values,
namely $0$, $4010$, $4020$, $4030$, and $4040$, evaluating Algorithm~\ref{alg1} and SCRiBLe \citep{abernethy2008competing}.
The artificial data are generated as follows.
The input set is constructed such that each element $x \in \mathbb{R}^d$ satisfies $\|x\|_2 \leq D$.
The loss vectors $\theta_1,\ldots,\theta_T$ are $d$-dimensional vectors sampled randomly before the experiment,
each with Euclidean norm bounded by $G$.
The $\nu$-self-concordant barrier is defined as
\(
\mathcal{R}(x) = -\log\!\left(1 - \frac{\|x\|^2}{D^2}\right),
\)
and the learning rate is set to
\(
\eta = \frac{\sqrt{\log(1/\delta)}}{4 d \sqrt{T}}.
\)

In our experiments, we set $D=5$, $G=3$, $d=5$, and $\nu=1$.
We consider the following values of the perturbation magnitude:
\(
\epsilon = 0,\; 0.005,\; 0.01,\; 0.015,\; 0.02.
\)
Given $\epsilon$, the perturbation is defined as
\(
\sigma(y_t) = \epsilon \sin\!\big((y_t^\top l)\pi\big),
\)
where $l \in \mathbb{R}^d$ is a random vector.
To further amplify the effect of perturbations near the boundary of the feasible set $\mathcal{K}$,
we increase the perturbation magnitude when $y_t$ approaches the boundary.
Specifically, we add a constant offset and redefine
\(
\sigma(y_t) \leftarrow \sigma(y_t) + 2.
\)
With this construction, for each value of $\epsilon$, the cumulative magnitude of the perturbation
is bounded by a corresponding constant $C$, namely,
\(
\sum_{t=1}^{T} \lvert \sigma_t (y_t)\rvert \le C.
\)
Note that the perturbation is not worst-case, which may explain why the result of Algorithm~\ref{alg1} exhibits little sensitivity to increasing $\epsilon$.
Each algorithm is run with a fixed time horizon of $T=2000$.
For each value of $C$, we repeat the experiment $10$ times for each algorithm
and report the average regret.
In the figure, the x-axis represents the value of $\epsilon$,
while the y-axis shows the corresponding regret.

As shown in the figure above, when $C=0$, SCRiBLe \citep{abernethy2008competing} outperforms Algorithm~\ref{alg1}.
One possible explanation is that, in our experimental setting, the optimal point lies on the boundary of the feasible set $\mathcal{K}$,
whereas Algorithm~\ref{alg1} restricts its actions to the shrunk domain $\mathcal{K}_\delta$.
Consequently, Algorithm~\ref{alg1} cannot select boundary points and incurs a larger loss in this regime.
As $\epsilon$ (and hence $C$) increases, the perturbation near the boundary becomes more pronounced,
which favors Algorithm~\ref{alg1} and allows it to outperform SCRiBLe.
We also note that the perturbation admits many possible choices,
and alternative perturbation designs may lead to different experimental outcomes.

\section{Conclusion}
In this work, we study a bandit optimization problem with non-convex and non-smooth losses, where each loss is composed of a linear component and an arbitrary adversarial perturbation revealed after the player’s action, subject to a global perturbation budget.
We develop a novel regret analysis tailored to this setting
and establish both expected and high-probability regret upper bounds, together with
a lower bound that highlights the intrinsic difficulty induced by the perturbations.
Our results extend classical bandit linear optimization to a strictly more general
regime and, through a standard online-to-offline conversion, suggest theoretical
guarantees for black-box optimization beyond the standard convex or smooth assumptions.
An interesting direction for future work is to extend our analysis to broader classes of losses, such as convex objectives with adversarial perturbations, and to further characterize intermediate regimes between stochastic and fully adversarial settings.

\section*{Statements and Declarations}

\begin{itemize}
\item Funding: This work was supported by WISE program (MEXT) at Kyushu
University,  JSPS KEKENHI Grant Numbers JP23K24905,
JP23K28038, and JP24H00685, respectively.

\item Conflict of interest/Competing interests: The authors have no competing interests to declare that are
relevant to the content of this article.
\item Ethics approval and consent to participate: Not applicable.
\item Consent for publication: All authors have read and approved the final manuscript. 

\item Data availability:
The data supporting the findings of this study were generated by the authors for simulation purposes.
No external datasets were used.
The data can be regenerated using the experimental procedure described in the paper.

\item Materials availability:
Not applicable.

\item Code availability:
The code used to generate the results in this study was developed by the authors and is available from the corresponding author upon reasonable request.

\item Author contribution: Z.C. was primarily responsible for the conception, design, theoretical development, analysis, and drafting of the manuscript.
K.H. and E.T. contributed to the development of the theoretical framework, provided guidance on the analysis, and critically revised the manuscript.
All authors reviewed and approved the final manuscript.

\end{itemize}

\appendix
\numberwithin{equation}{section}





\bibliography{main, hatano}

@inproceedings{flaxman-etal:soda05,
   author = {Abraham D Flaxman and Adam Tauman Kalai and H Brendan Mcmahan},
   doi = {10.5555/1070432},
   booktitle = {Proceedings of the sixteenth annual ACM-SIAM symposium on Discrete algorithms},
   pages = {385 - 394},
   title = {Online convex optimization in the bandit setting: gradient descent without a gradient},
   year = {2005},
}

@article{lattimore:msl20,
   author = {Tor Lattimore},
   doi = {10.4171/msl/17},
   issn = {2520-2316},
   issue = {3},
   journal = {Mathematical Statistics and Learning},
   month = {10},
   pages = {311-334},
   title = {Improved regret for zeroth-order adversarial bandit convex optimisation},
   volume = {2},
   year = {2020},
}

@article{abbasi2011improved,
  title={Improved algorithms for linear stochastic bandits},
  author={Abbasi-Yadkori, Yasin and P{\'a}l, D{\'a}vid and Szepesv{\'a}ri, Csaba},
  journal={Advances in neural information processing systems},
  volume={24},
  year={2011}
}

@article{lee2020bias,
  title={Bias no more: high-probability data-dependent regret bounds for adversarial bandits and mdps},
  author={Lee, Chung-Wei and Luo, Haipeng and Wei, Chen-Yu and Zhang, Mengxiao},
  journal={Advances in neural information processing systems},
  volume={33},
  pages={15522--15533},
  year={2020}
}

@inproceedings{abernethy2008competing,
  title={Competing in the Dark: An Efficient Algorithm for Bandit Linear Optimization.},
  author={Abernethy, Jacob D and Hazan, Elad and Rakhlin, Alexander},
  booktitle={COLT},
  pages={263--274},
  year={2008}
}

@book{nesterov1994interior,
  title={Interior-point polynomial algorithms in convex programming},
  author={Nesterov, Yurii and Nemirovskii, Arkadii},
  year={1994},
  publisher={SIAM}
}

@article{nemirovski2004interior,
  title={Interior point polynomial time methods in convex programming},
  author={Nemirovski, Arkadi},
  journal={Lecture notes},
  volume={42},
  number={16},
  pages={3215--3224},
  year={2004},
  publisher={Citeseer}
}

@article{hazan2016introduction,
  title={Introduction to online convex optimization},
  author={Hazan, Elad and others},
  journal={Foundations and Trends{\textregistered} in Optimization},
  volume={2},
  number={3-4},
  pages={157--325},
  year={2016},
  publisher={Now Publishers, Inc.}
}

@inproceedings{awerbuch2004adaptive,
  title={Adaptive routing with end-to-end feedback: Distributed learning and geometric approaches},
  author={Awerbuch, Baruch and Kleinberg, Robert D},
  booktitle={Proceedings of the thirty-sixth annual ACM symposium on Theory of computing},
  pages={45--53},
  year={2004}
}

@inproceedings{mcmahan2004online,
  title={Online geometric optimization in the bandit setting against an adaptive adversary},
  author={McMahan, H Brendan and Blum, Avrim},
  booktitle={Learning Theory: 17th Annual Conference on Learning Theory, COLT 2004, Banff, Canada, July 1-4, 2004. Proceedings 17},
  pages={109--123},
  year={2004},
  organization={Springer}
}

@inproceedings{bubeck2012towards,
  title={Towards minimax policies for online linear optimization with bandit feedback},
  author={Bubeck, S{\'e}bastien and Cesa-Bianchi, Nicolo and Kakade, Sham M},
  booktitle={Conference on Learning Theory},
  pages={41--1},
  year={2012},
  organization={JMLR Workshop and Conference Proceedings}
}

@article{ghai2022non,
  title={Non-convex online learning via algorithmic equivalence},
  author={Ghai, Udaya and Lu, Zhou and Hazan, Elad},
  journal={Advances in Neural Information Processing Systems},
  volume={35},
  pages={22161--22172},
  year={2022}
}

@inproceedings{agarwal2019learning,
  title={Learning in non-convex games with an optimization oracle},
  author={Agarwal, Naman and Gonen, Alon and Hazan, Elad},
  booktitle={Conference on Learning Theory},
  pages={18--29},
  year={2019},
  organization={PMLR}
}

@inproceedings{gao2018online,
  title={Online learning with non-convex losses and non-stationary regret},
  author={Gao, Xiand and Li, Xiaobo and Zhang, Shuzhong},
  booktitle={International Conference on Artificial Intelligence and Statistics},
  pages={235--243},
  year={2018},
  organization={PMLR}
}

@article{yang2018optimal,
  title={An optimal algorithm for online non-convex learning},
  author={Yang, Lin and Deng, Lei and Hajiesmaili, Mohammad H and Tan, Cheng and Wong, Wing Shing},
  journal={Proceedings of the ACM on Measurement and Analysis of Computing Systems},
  volume={2},
  number={2},
  pages={1--25},
  year={2018},
  publisher={ACM New York, NY, USA}
}

@article{amani2019linear,
  title={Linear stochastic bandits under safety constraints},
  author={Amani, Sanae and Alizadeh, Mahnoosh and Thrampoulidis, Christos},
  journal={Advances in Neural Information Processing Systems},
  volume={32},
  year={2019}
}

@article{abernethy2012interior,
  title={Interior-point methods for full-information and bandit online learning},
  author={Abernethy, Jacob D and Hazan, Elad and Rakhlin, Alexander},
  journal={IEEE Transactions on Information Theory},
  volume={58},
  number={7},
  pages={4164--4175},
  year={2012},
  publisher={IEEE}
}

@inproceedings{bartlett2008high,
  title={High-probability regret bounds for bandit online linear optimization},
  author={Bartlett, Peter and Dani, Varsha and Hayes, Thomas and Kakade, Sham and Rakhlin, Alexander and Tewari, Ambuj},
  booktitle={Proceedings of the 21st Annual Conference on Learning Theory-COLT 2008},
  pages={335--342},
  year={2008},
  organization={Omnipress}
}

@article{rodemann2024reciprocal,
  title={Reciprocal learning},
  author={Rodemann, Julian and Jansen, Christoph and Schollmeyer, Georg},
  journal={Advances in Neural Information Processing Systems},
  volume={37},
  pages={1686--1724},
  year={2024}
}

@inproceedings{ito2023best,
  title={Best-of-three-worlds linear bandit algorithm with variance-adaptive regret bounds},
  author={Ito, Shinji and Takemura, Kei},
  booktitle={The Thirty Sixth Annual Conference on Learning Theory},
  pages={2653--2677},
  year={2023},
  organization={PMLR}
}

@article{ito2023exploration,
  title={An exploration-by-optimization approach to best of both worlds in linear bandits},
  author={Ito, Shinji and Takemura, Kei},
  journal={Advances in Neural Information Processing Systems},
  volume={36},
  pages={71582--71602},
  year={2023}
}

@inproceedings{cheng2025adversarial,
  title={Adversarial bandit optimization for approximately linear functions},
  author={Cheng, Zhuoyu and Hatano, Kohei and Takimoto, Eiji},
  booktitle={International Conference on Discovery Science},
  pages={495--509},
  year={2025},
  organization={Springer}
}


\end{document}